\documentclass[letterpaper, 10 pt, conference]{ieeeconf}  

\IEEEoverridecommandlockouts                              

\usepackage[T1]{fontenc}
\usepackage{graphicx}
\usepackage{amsmath}
\usepackage{amsfonts} 
\usepackage{amssymb} 
\usepackage{subcaption}
\usepackage{mathtools} 
\usepackage{graphicx} 
\usepackage{xcolor} 
\usepackage{acronym}
\usepackage{cite}
\let\labelindent\relax 
\usepackage[inline]{enumitem}

\usepackage{balance}
\usepackage[none]{hyphenat}
\usepackage{acronym}

\usepackage{tikz}
\usetikzlibrary{trees}
\def\Nat{{\mathbb{N}}}

\def\X{{\mathcal{X}}}
\def\L{{\mathcal{L}}}
\def\G{{\mathcal{G}}}
\def\T{{\mathcal{T}}}
\def\R{{\mathcal{R}}}

\def\F{{\mathcal{F}}}

\def\M{{\mathcal{M}}}

\def\lessNat{{<\Nat}}

\def\subder{{\text{der}}}

\def\subtask{{\text{task}}}

\def\hissigma{{\tilde{\sigma}}}

\def\imap{{\kappa}}

\def\hisu{\tilde{u}}

\def\fMP{{\diamond_f}}
\def\fTr{{\bar{f}}}

\def\ginc{{\text{ginc}}}
\def\inc{{\text{inc}}}
\def\adj{{\text{adj}}}
\def\con{{\text{car}}}

\def\move{{\text{move}}}
\def\picktext{{\text{pick}}}

\def\placeontext{{\text{placeon}}}

\def\MP{{\text{MP}}}

\newcommand{\cat}{{}^\frown}
\def\preimg{{^{-1}}}

\renewcommand{\emptyset}{\varnothing}

\def\ie{\textit{i.e.,}}
\def\eg{\textit{e.g.,}}

\def\captionedges{Solid lines are used for inclusion relations, dashed lines for adjacency, and double lines for carrier.}

\acrodef{SG}[SG]{Scene Graph}

\acrodef{TAMP}[TAMP]{Task and Motion Planning}

\acrodef{ITS}[ITS]{Information Transition System}

\acrodef{Ispace}[I-space]{Information Space}
\acrodef{Istate}[I-state]{Information State}

\DeclareMathOperator*{\moveto}{{\tt moveto}}
\DeclareMathOperator*{\pick}{{\tt pick}}

\DeclareMathOperator*{\placeon}{{\tt placeon}}

\newtheorem{claimenv}{\bf{Claim}}
\newtheorem{corollaryenv}{\bf Corollary}
\newtheorem{lemmaenv}{\bf Lemma}
\newtheorem{exampleenv}{\bf Example}
\newtheorem{remarkenv}{\bf Remark}
\newtheorem{propositionenv}{\bf Proposition}
\newtheorem{mydefenv}{\bf Definition}

\newenvironment{lemma}{\vspace{0.05in}\begin{lemmaenv}\em}{\end{lemmaenv}\vspace{0.05in}}

\newenvironment{corollary}{\vspace{0.05in}\begin{corollaryenv}\em}{\end{corollaryenv}\vspace{0.05in}}
\newenvironment{example}{\vspace{0.05in}\begin{exampleenv}}{\end{exampleenv}\vspace{0.05in}}

\newenvironment{proposition}{\vspace{0.05in}\begin{propositionenv}}{\end{propositionenv}\vspace{0.05in}}
\newenvironment{definition}{\vspace{0.05in}\begin{mydefenv}}{\end{mydefenv}\vspace{0.05in}}

\def\figvspace{}

\title{\LARGE \bf
An Information-Space Perspective to \\
Scene Graph Sufficiency for Robotic Task Planning
}

\author{Başak Sakçak$^{1}$ and Francesco Verdoja$^{2}$
\thanks{F.\ Verdoja was supported by the Research Council of Finland (decision 354909) during this work.}
\thanks{$^{1}$Dept.\ of Advanced Computing Sciences, Maastricht University, the Netherlands ({\tt\small basak.sakcak@maastrichtuniversity.nl})}%
\thanks{$^{2}$School of Electrical Engineering, Aalto University, Finland ({\tt\small francesco.verdoja@aalto.fi})}%
}

\begin{document}

\maketitle
\thispagestyle{empty}
\pagestyle{empty}

\begin{abstract}
    Planning in complex environments requires task specifications grounded in representations that capture objects, relations, and affordances; scene graphs meet this need, but their size in large environments hinders efficient planning. While task-aware pruning and hierarchical abstractions have been explored, a general, task-centric formalization of what constitutes a sufficient scene graph for planning remains open. This paper provides such a formalization by modeling planning over scene graphs within an information-spaces framework through the definition of scene graph transition systems and relevant action semantics for navigation and manipulation. We then introduce derived scene graphs via information mappings that merge and prune nodes and induce quotient transition systems augmented with motion primitives to capture higher-level actions over merged graph nodes. Sufficiency is characterized by two conditions: (i) the information mapping yields a deterministic quotient, and (ii) the task is well-posed over derived traces, ensuring plans found on the derived model are feasible on the maximal system. We illustrate the framework using a task over an example environment, showing both sufficient and insufficient reduced scene graphs.
\end{abstract}

\section{Introduction}
Robots typically achieve desired behavior in the physical world by planning sequences of actions in advance.
Planning, in turn, requires a model that captures the objects in the physical world, their attributes, and the relations and affordances that constrain possible robot actions. Over this model, the desired behavior can be specified and effects of the robot actions can be evaluated.
Planning directly on raw world models (\eg{} dense metric maps or point clouds) becomes quickly intractable, while purely symbolic descriptions lack the contextual detail needed to ensure realizable, context-appropriate plans.

Scene graphs have emerged as a compelling middle ground: by modeling environments as hierarchical graphs of spatial entities and their relations, they provide a semantically rich abstraction that grounds planning in the actual state of the world, and can be constructed online from sensor data to autonomously instantiate planning domains \cite{hughes2022hydra,ray2024task}.
Yet real environments induce large scene graphs, significantly increasing planning complexity~\cite{agia2022taskography}. Existing scene-graph planning approaches show that smaller, task-relevant graphs improve planning~\cite{ray2024task, rana_sayplan_2023, nguyen_event-grounding_2026,agia2022taskography}, but they largely rely on heuristics or Large Language Models (LLMs), offering empirical rather than principled guarantees about preserving plan existence and feasibility. This paper focuses on the complementary formal aspect of determining what information a derived scene graph must preserve so that plans found on it correspond to feasible plans on the maximal scene graph.

In parallel with this scene-graph literature, a line of work on information spaces studies what encoding of information given as a robot's interaction history of action-observation pairs is needed for tasks to remain solvable~\cite{sakcak2024mathematical,mcfassel2023reactivity}. However, a gap exists between this formalism and its application to concrete problems.
Borrowing ideas from \cite{sakcak2024mathematical}, we address the problem of ``simplifying'' a scene graph representation and propose a formal structure over which questions about the sufficiency of a simplified scene graph in terms of determining a plan feasible for the original representation can be understood in a precise way.

The main contribution of this paper is a framework for reasoning about sufficiency of scene graphs for robotic task planning.
First we extend the framework in \cite{sakcak2024mathematical} and introduce conditions under which an internal representation is sufficient for planning purposes, and when motion primitives are used. Then, we connect these conditions to planning over scene graphs.
We define scene graph transition systems, involving navigation and manipulation actions. Then, we formalize derived scene graphs as quotient models induced by many-to-one information mappings, with motion primitives accounting for high-level actions. We then provide conditions under which plans found over derived systems remain feasible for the maximal system.


\section{Related Literature}
Mapping semantic information enables complex planning in many robotics domains, from autonomous driving to mobile manipulation. Beyond augmenting traditional 2D/3D maps, 3D Scene Graphs (3DSGs) have emerged as metric-semantic representations whose nodes denote spatial concepts (rooms, objects) and edges encode relations (inclusion, adjacency) \cite{hughes_foundations_2024, kuhn_relational_2026}. They provide memory-efficient, perception-driven semantic understanding and have been demonstrated in real-time operation on robots, \eg{} Hydra \cite{hughes2022hydra,hughes_foundations_2024}. Extensions broaden semantic content with open-vocabulary features (ConceptGraphs \cite{gu_conceptgraphs_2024}) and address long-term mapping \cite{nguyen_react_2025,schmid_khronos_2024}.

Scene graphs are increasingly being adopted as representations for task and motion planning (TAMP). Two main directions appear: leveraging scene-graph structure and hierarchy to build planning domains or to plan iteratively across levels \cite{ray2024task,dai_optimal_2024,ejaz_situationally-aware_2025,ni_grid_2024}, and using large language models (LLMs) that consume the scene graph as context \cite{rana_sayplan_2023,ni_grid_2024}. Because planning on full graphs is complex, most of these methods decompose them through either: high-to-low abstraction planning that focuses on nodes relevant to the current plan \cite{dai_optimal_2024,ni_grid_2024}; heuristics that compress dense subgraphs (\eg{} aggregating equivalent nodes) \cite{ray2024task,agia2022taskography}; and iterative subgraph expansion, sometimes guided by LLMs \cite{rana_sayplan_2023}.
Some existing methods (\eg{}~\cite{agia2022taskography}), provide guarantees in specific symbolic planning settings, while many heuristic, learned, and LLM-guided reductions are evaluated empirically. Complementary to these methods, we provide a formal framework for determining when a derived scene graph induced by pruning (but also merging) preserves information needed for planning.

More broadly, solving robotic planning problems requires choosing a model of the robot–environment dynamics, which also determines plan existence and feasibility in the original system. A key line of work asks what information from action–observation histories must be retained for tasks to be solvable: histories are mapped to \textit{information states} (equivalence classes), yielding minimally sufficient quotients that can simulate history-based policies \cite{sakcak2024mathematical,subramanian2022approximate,mcfassel2023reactivity,sakcak2025minimally}.
These ideas align with formal methods, where abstractions preserve properties via simulation and refinement \cite{kress2018synthesis}. In TAMP literature, PDDL~\cite{Ghallab1998pddl} abstracts geometric and differential constraints into predicates and discrete actions that must be grounded by motion planning; because even few objects and predicates induce large state spaces, structure-preserving simplifications are pursued \cite{horvcik2022homomorphisms,allen2026simplifying}.

Together, these works show that scene graphs are useful planning representations but abstraction is necessary for scalability. However, existing scene-graph planning approaches typically evaluate abstractions empirically. In contrast, information-space and formal-abstraction methods investigate whether a representation preserves the necessary task-related information for planning or plan execution. Here, we bridge these directions by casting scene-graph reduction as an information mapping and by giving conditions under which a derived scene graph is sufficient for planning.


\section{Sufficient Models for Planning}\label{sec:theory}
\subsection{Notation}
For a given set $A$, we write $A^{\leq\Nat}$ to denote the set of all finite length sequences of the elements of $A$. Given a function $f:A \rightarrow B$, $f(A)$ denotes its image and the preimage of $b \in B$ under $f$ is denoted by $f\preimg(b)$.
Given a non-injective function $g:A \rightarrow B$, $[a]_g := \{ a' \in A \mid g(a)=g(a') \}$ denotes an equivalence class induced by $g$. Finally, we will write $A/g:=\{ [a]_g \mid a \in A \}$ to denote the set of all equivalence classes induced by $g$. Given two functions $g_1:A\rightarrow B_1$ and $g_2:A\rightarrow B_2$, \emph{$g_1$ is a refinement of $g_2$} if the partitioning over $A$ induced by $g_1$ refines that of $g_2$.

A transition system is a triple $\mathcal{S}=(S,\Lambda,T)$, in which $S$ is the set of states, $\Lambda$ is the set of names for the outgoing transitions, and $T \subset S \times \Lambda \times S$ is a ternary relation describing the transitions. If for each $(s,\lambda) \in S \times \Lambda$ there is a unique $s' \in S$ such that $(s,\lambda,s') \in T$, then we will write this system as $(S, \Lambda, \tau)$ in which $\tau: S\times \Lambda \rightarrow S$ is a function. Given a transition system $\mathcal{S}=(S,\Lambda,\tau)$ a (finite) \emph{trace} of $\mathcal{S}$ is a sequence of pairs $(s_1, \lambda_1, \dots, s_{N-1},\lambda_{N-1}, s_N)$ for some $N$, satisfying that $s_i \in S$ and $\lambda_i \in \Lambda$ for all $i=1,\dots,N$, and respect $\tau$ such that for all $i=1,\dots,N-1$, $s_{i+1}=\tau(s_i,\lambda_i)$.
Given a labeling function $\sigma: S \rightarrow L$, in which $L$ is a set of labels, the quotient of $\mathcal{S}$ by $\sigma$ is the transition system $(S/\sigma,\Lambda,T/\sigma)$, denoted by $\mathcal{S}/\sigma$, in which $S/\sigma=\{[s]_\sigma \mid s\in S\}$ and $T/\sigma=\{\left([s]_\sigma, \lambda, [s']_\sigma\right) \mid (s, \lambda, s') \in T\}$.

\subsection{Planner Perspective}
A robot embedded in and interacting with an environment is modeled as two coupled transition systems, named \emph{internal} and \emph{external} systems in~\cite{sakcak2024mathematical}.
The external system corresponds to the physical environment, including the robot body, and the internal system represents the perspective of a decision maker.
Let $\X=(X,U,f)$ be a transition system representing the external system, in which, $X$ and $U$ are the external state and action spaces, $f: X \times U \rightarrow X$ is a state transition function such that $f(x,u)$ is the state reached after taking action $u$ at state $x$. In this work, we extend the framework presented in~\cite{sakcak2024mathematical} by introducing the \emph{planner perspective} for the internal system.
Given a task description, we consider a planner that systematically explores the space of all finite sequences of actions over a given model and determines one that satisfies the given task.
Tasks are modeled over the finite external system traces.

\begin{definition}[Task]
    Given $\X$, let$\Omega_\X$ denote the set of all traces of $\X$.
    A task $T$ for an external system $\X$ is the subset $T \subseteq \Omega_\X$.
\end{definition}

Given a task, a plan is determined so that the robot exhibits the desired behavior, in which case it is called a \emph{feasible plan for task $T$}. In this work, we focus on plans that are represented as a sequence of actions, as this is typical for many planning problems, including many for which scene graphs have been proposed, \eg{}~\cite{ray2024task,rana_sayplan_2023}.

\begin{definition}[Feasible plan for $T$]
    Given an external system $\X=(X,U,f)$ initialized at $x_1 \in X$, let $\omega_{\hisu(x_1)} = (x_1, u_1, x_2, \dots, u_{N-1},x_N) \in \Omega_\X$ be a trace corresponding to a plan $\hisu(x_1) = (u_1, u_2, \dots, u_{N-1})$ of some length $N-1$. A plan $\hisu(x_1)$ is called a \emph{feasible plan for $T$} if $\omega_{\hisu(x_1)} \in T$.
\end{definition}

Given a task $T$ for an external system $\X$, the objective of a planner is to determine a feasible plan for $T$. However, typically, the planner is given a \emph{model} of the external system over which also the task $T$ is transcribed. Let $\M=(X_M,U,f_M)$ be a \emph{model of the external system}, in which $X_M$ is the set of states, $U$ is the set of actions, and $f_M$ is a state transition function. Given $\M$, an initial state $x_1\in X_M$, and a task $T_M \subseteq \Omega_\M$, a planner systematically explores $X_M$ and the space of action sequences consistent with $\M$.
A solution is found if a particular sequence of actions results in a trace that belongs in $T_M$.
Two questions arise: (i) can the task $T$ be expressed over the traces of $\M$? (ii) is a plan feasible for $T_M$ also feasible for $T$? These questions are addressed in the following subsection.

\subsection{Sufficient Models for Planning}
To determine conditions under which feasible plans determined by a planner using a particular model are feasible for the external system as well, we focus on models that are quotient systems derived from an external system $\X=(X,U,f)$ through a labeling function $\imap: X \rightarrow X_M$ named an \emph{information mapping (I-map)}. Given $\imap$, the quotient of $\X$ by $\imap$ is the transition system $\M:=(X/\imap, U, f/\imap)$.
At the highest level, a model of the external system is the external system itself, that is, $\imap$ is a bijection and $\M=\X$. However, in most cases this model is prohibitive for search, and we are interested in determining a ``simpler'' one as a quotient system obtained by a many-to-one mapping $\imap$. A crucial requirement on this quotient system is that it should still allow determining a feasible plan for task $T$. Note that the quotient system may not be deterministic as this depends on the selected labeling function $\imap$.
Suppose for all $x,s,x',s' \in X$ and all $u\in U$, the following holds:
\begin{equation}\label{eq:sufficiency}
    \begin{aligned}
        \imap(x)=\imap(s) \,\land\, x'=f(x, u) \,\land \, s'=f(s, u) \implies \\
        \imap(x')=\imap(s').
    \end{aligned}
\end{equation}
Then, the resulting quotient system $(X/\imap, U, f/\imap)$ is deterministic as well~\cite{sakcak2024mathematical}.
A labeling function $\imap$ is called \emph{sufficient} if it satisfies the implication given in \eqref{eq:sufficiency}.
Thus, if $\imap$ is sufficient, the labels can be uniquely determined when limited to the quotient system.

When using a \emph{derived system} for planning, it is natural to ask whether the task can still be expressed over the traces of this new system. Given $\X=(X,U,f)$ and a derived system $\M=(X_M,U,f_M):=\X/\imap$,
let $\Omega_\X$ and $\Omega_{\M}$ be the corresponding sets of traces. Furthermore, let $c_\imap:\Omega_\X \rightarrow \Omega_{\M}$ be a mapping, named \emph{I-map condenser}, under which
\begin{equation}
    (x_1,u_1,\dots,x_N) \mapsto (\imap(x_1),u_1,\dots,\imap(x_N)).
    \label{eq:cImap}
\end{equation}
A task $T$ is called \emph{well-posed under $c_{\imap}$} if $c_{\imap}\preimg(c_{\imap}(T))=T$~\cite{sakcak2025limits}.
Informally, well-posedness means that a task expressed over $\Omega_\X$ can similarly be expressed over $\Omega_{\M}$ without losing any task-related information. Thus, because task accomplishment information can be determined solely by the derived system traces without ambiguities, we can say that the derived system is rich enough for planning purposes.

\begin{proposition}\label{prop:poseable_feasible} Let $\X=(X, U, f)$ be an external system and $\M=(X_M,U,f_M):=\X/\imap$ be a model. Let $T \subseteq \Omega_\X$ be a task. Suppose $\imap$ is sufficient and $T$ is well-posed under $c_\imap$. Let $T_\M = c_\imap(T)$ be the task posed over $\M$ such that $T_\M \subseteq \Omega_\M$. Then, any feasible plan for $T_\M$ is also a feasible plan for $T$.
\end{proposition}
\begin{proof}
    Let $\hisu(s_1)=(u_1,\dots,u_N)$ be a feasible plan for $T_M$ from $s_1=\imap(x_1)$ and let $\omega_\X$ be the corresponding trace of $\X$ generated by the same sequence from $x_1$. Since $\imap$ is sufficient, the condensed trace $c_\imap(\omega_\X)$ is the unique trace $\omega_\M$ generated in $\M$ from $\imap(x_1)$. Because $\hisu(s_1)$ is feasible for $T_\M=c_\imap(T)$, $\omega_\M \in c_\imap(T)$. Because $T$ is well-posed under $c_\imap$, $\omega_\X \in c_\imap^\preimg(c_\imap(T))=T$. Therefore, the sequence of actions $\hisu(s_1)$ is feasible for $T$ from $x_1 \in X$.
\end{proof}

Let $c_{\subtask} : \Omega_\X \rightarrow \{0,1\}$ such that $c_\subtask(\omega)=1$ if $\omega \in T$ and $0$ otherwise.
\begin{lemma}
    Task $T$ is well-posed under $c_\imap$ if and only if $c_\imap$ is a refinement of $c_\subtask$.
    \label{lem:refinement}
\end{lemma}
\begin{proof}
    ($\implies$ direction) Suppose $T$ is well-posed under $c_\imap$. For contradiction, suppose $c_\imap$ does not refine $c_\subtask$. Then there exist $\omega,\omega' \in \Omega_\X$ such that $c_\imap(\omega)=c_\imap(\omega')$, but $c_\subtask(\omega)\not=c_\subtask(\omega')$. Let $\omega \in T$ and $\omega' \not\in T$. Since $c_\imap(\omega)=c_\imap(\omega')$ and $\omega \in T$ we have $\omega' \in c_\imap\preimg(c_\imap(T))$ which contradicts with well-posedness.
    ($\impliedby$ direction) For contradiction suppose $c_\imap\preimg(c_\imap(T)) \not= T$. Then there exist $\omega \in T$ and $\omega' \not\in T$ with $c_\imap(\omega)=c_\imap(\omega')$.
    Refinement then implies that $\omega,\omega' \in c_\subtask\preimg(1)$, contradicting $\omega'\not\in T$.
\end{proof}

The following definition combines two requirements for a useful model. Sufficiency of $\imap$ guarantees deterministic behavior of the model transitions and sufficiency of $c_\imap$, while well-posedness guarantees that task satisfaction is determined completely by the model trace.
\begin{definition}[Sufficient model for planning]\label{def:sufficient_model} Given an external system $\X$, a model $\M = \X/\imap$ is called \emph{sufficient for planning for task $T$} if (i) $\imap$ is sufficient, and (ii) the task $T$ is well-posed under $c_\imap$.
\end{definition}

\subsection{Motion Primitives}\label{sec:motion_prim}
The quotient structure introduced in the previous section assumes that the original and the derived system share the same set of actions. However, often a set of higher-level actions (motion primitives) that correspond to a sequence from set $U$ is used, typically parametrized by the system states. For example, an action such as move to a region may require executing a sequence of lower-level actions whose intermediate external states are not represented in the model. The next definition formalizes this notion.

\begin{definition}[Motion primitive] Let $\X=(X,U,f)$ be an external system. A motion primitive $\sigma$ is a mapping $\sigma: X \rightarrow U^\lessNat$ assigning to every state $x\in X$ a finite action sequence $\sigma(x)=(u_1,\dots,u_N)$. Starting from some $x_1=x$, the execution of $\sigma(x)$ induces the finite trace
    $$\bar{f}(x,\sigma) := \left(x_1, u_1, x_2, \dots, u_N, x_{N+1}\right),$$
    in which $x_{i+1} = f(x_i, u_i)$. The corresponding endpoint transition is $\diamond_f(x,\sigma):=x_{N+1}.$
\end{definition}

We can consider a set of motion primitives to be the new set of actions for the external system. Let $\Sigma$ denote the set of motion primitives.
Using $\Sigma$ as the set of actions, the system restricted by the motion primitives is given by $\X_\Sigma = (X, \Sigma, \diamond_f)$ and we use ${\Omega}_{\X_\Sigma}$ to denote its traces.

Given a sequence of actions $\tilde{\sigma} \in (\Sigma)^N$ for some $N$, the respective trace of $\X_\Sigma$ is given by
$$\omega_{\tilde{\sigma}}=(x_1, \sigma_1, x_2, \sigma_2,\dots, \sigma_N, x_{N+1}) \in \Omega_{\X_\Sigma},$$
in which $x_{i+1}=\fMP(x_i,\sigma_i)$, $i=1,\dots,N$. Each motion primitive $\sigma_i$ induces a finite trace $\fTr(x_i,\sigma_i)$ of $\X$. Hence, the corresponding trace in $\X$ is obtained by concatenating the subtraces induced by the individual motion primitives,
\begin{equation*}
    \bar{\omega}_{\tilde{\sigma}} = \fTr(x_1,\sigma_1)\cat\fTr(x_2,\sigma_2)\cat\dots\cat \fTr(x_N,\sigma_N),
\end{equation*}
in which $\cat{}$ denotes the concatenation of traces with the common boundary states appearing only once.
Let
$$ \bar{\Omega}_{\X_\Sigma}:=\{ \bar{\omega}_{\hissigma} \in \Omega_{\X} \mid \hissigma \in (\Sigma)^N, N<\infty \} \subseteq \Omega_\X$$
denote the set of external system traces generated by finite sequences of motion primitives. Since each $\bar{\omega}_{\hissigma}$ is defined through its generating primitive sequence $\hissigma$, its decomposition into motion primitive induced subtraces is fixed by construction. We then define the motion-primitive condenser
$ c_\MP : \bar{\Omega}_{\X_\Sigma} \rightarrow {\Omega}_{\X_\Sigma}$ by $c_\MP(\bar{\omega}_{\tilde{\sigma}}) = \omega_{\tilde{\sigma}}$. Since restricting the action set to motion primitives may exclude some traces of $\Omega_\X$, we define the task considered by this restricted system as $T_\Sigma := T \cap \bar{\Omega}_{\X_\Sigma}$.

For a set $X_0 \subseteq X$ of admissible initial states, we say that $\Sigma$ is \emph{complete for $T$} if $\forall x_1 \in X_0$, $T_{x_1} \not= \emptyset \implies (T_\Sigma)_{x_1}\not=\emptyset$, in which the subscript $x_1$ denotes the respective traces beginning at $x_1$. If $\Sigma$ is complete for $T$, then existence of feasible plan from $x_1$ over $\X$ implies that there exists a feasible plan from $x_1$ over $\X_\Sigma$.


\section{Scene Graph Transition Systems}
In this section, we introduce the notion of \emph{scene graph transition system} as an application of the theory presented in Sec. \ref{sec:theory}. When defining the scene graph, we follow the conventions used in previous work~\cite{hughes2022hydra,tang_openobject-nav_2025}. We will also introduce a sample set of actions relevant to simple navigation and manipulation, similar to~\cite{ray2024task}.

\subsection{Scene graph representation}
A scene graph $G=(V,E)$ maps spatio-semantic relations among scene elements via edges $E$ over nodes $V$. Nodes represent spatial elements (\eg{} tables, robots, bedrooms) with spatial features (pose/configuration, bounding box) and semantic features (class, affordances, color). Following \cite{hughes2022hydra}, we use five node categories: \emph{building}, \emph{rooms}, \emph{places}, \emph{objects}, and \emph{agents}. A \emph{place} is a region within a room obtained by space decomposition, \eg{} GVD \cite{oleynikova_sparse_2018} or 2D polygons from meshes \cite{ray2024task}, such that each object/agent lies entirely in one place and all navigable positions within a place are path-connected. We denote building, room, place, and object, and agent node sets as $V_b$, $V_r$, $V_p$, $V_o$, and $V_a$, respectively. In this work, we take $V_a:=\{v_a\}$, in which $v_a$ is the node representing the robot.
Many scene graphs include a 3D mesh \cite{hughes2022hydra} or point cloud \cite{tang_openobject-nav_2025} as the global spatial frame; we omit this here since it mainly supports localization and motion planning rather than abstract reasoning or task planning. We consider three types of edges, encoding \textit{inclusion}, \textit{adjacency}, and \textit{carrier} relations, though additional relations (\eg{} functional \cite{rotondi_fungraph_2025} or event \cite{nguyen_event-grounding_2026}) are possible.

\begin{definition}[General Inclusion Relation]\label{def:rel_ginc}
    A \emph{general inclusion relation} is an asymmetric and transitive binary relation denoted by $\R_\ginc \subseteq V \times V$. We say $v_j$ is included in $v_i$ if $(v_i, v_j) \in \R_\ginc$. Let $v_{s}$ be the root node which is the largest spatial category such that all nodes in $V \setminus \{v_s\}$ are included in $v_s$. Then, $\R_\ginc$ satisfies that for any $v_j \in V \setminus \{v_s\}$ exists a $v_i \in V$ such that $(v_i, v_j) \in \R_\ginc$.
\end{definition}
\begin{definition}[Reduced Inclusion Relation]\label{def:rel_inc}
    A reduced inclusion relation, alternatively, an \textit{inclusion relation}, is a binary relation $\R_\inc \subseteq \R_\ginc$ obtained by removing from $\R_\ginc$ the pairs $(v_i,v_j)$ satisfying that there exists a $v_k$ such that $(v_i,v_k)\in \R_\ginc$ and $(v_k,v_j)\in \R_\ginc$.
\end{definition}

When considering spatial elements at progressively finer resolutions (\eg{} from buildings to objects), inclusion relations are used to introduce a spatial hierarchy (\eg{} the robot being in a specific place). We assume each spatial element $v_j$ to only have one parent, \ie{} $\forall v_j \in V \setminus \{v_s\}, \{v_i \mid (v_i , v_j) \in \R_\inc\}$ is a singleton.
Under this assumption, inclusion relations (Definition~\ref{def:rel_inc}) form a hierarchical tree $\T_\inc = (V,\R_\inc)$.
For concreteness, in this paper, we use a tree with depth three, with nodes at each depth representing respectively the building ($V_0=V_b$, consisting of the root node $v_s$ alone), rooms ($V_1=V_r$), places ($V_2=V_p$), and objects and agents ($V_3=V_o\cup V_a$).

Let $\mathcal{F}(v_p)$ denote the free configuration space associated with a place node $v_p \in V_p$. We assume that each $\mathcal{F}(v_p)$ is path connected.

\begin{definition}[Adjacency Relation]\label{def:rel_adj}
    An adjacency relation is a symmetric irreflexive relation $\R_{\adj} \subseteq V_p \times V_p$. Two place nodes $v_i,v_j \in V_p$ are adjacent, written $(v_i,v_j) \in \R_{\adj}$, if there exists a path contained in $\mathcal{F}(v_i) \cup \mathcal{F}(v_j)$, or in its closure, connecting a point in $\mathcal{F}(v_i)$ to a point in $\mathcal{F}(v_j)$.
\end{definition}

\begin{definition}[Carrier Relation]\label{def:rel_conn}
    A carrier relation is an irreflexive and acyclic relation over $V_a$ and $V_o$, \ie{} $\R_{\con} \subseteq (V_o\cup V_a) \times V_o$.
    We say that $v_i$ is carrying $v_j$ if $(v_i,v_j) \in \R_{\con}$. The relation represents a kinematic link, in such a way that if $v_i$ moves, then $v_j$ moves with it.
\end{definition}

\begin{definition}[Scene Graph]\label{def:scene_graph}
    A \emph{scene graph} is an edge-labeled directed graph $G=(V,E)$ in which $E=\R_\inc \cup \R_\adj \cup \R_\con$ together with an edge labeling function $\ell_e: E \rightarrow \L $ that attributes each edge a label from the set $\L=\{\inc,\adj,\con\}$ consistent with the relation that generated the edge.
\end{definition}

This definition makes explicit the scene graph structure that prior works use implicitly (\eg{}~\cite{hughes2022hydra,ray2024task,rana_sayplan_2023,tang_openobject-nav_2025,nguyen_fast_2026}); to our knowledge, this is the first formal statement of it.
In the following, we will use $E(G)$ and $V(G)$ to denote the edges and nodes of $G$, respectively. We will use $(v_i,v_j)_\ell$ to denote the label $\ell \in \L$ of the edge connecting $(v_i,v_j)$.

\begin{figure}
    \centering
    \begin{subfigure}{0.3\linewidth}
    \centering
    \begin{tikzpicture}[
        scale=0.4,
        every node/.style={scale=0.8}
        ]
        \begin{scope}[rotate=-90]
            \node at (0.5,-0.5) {$r_1$};
            \node at (4,-0.5) {$r_2$};
            \node at (7,-0.5) {$r_3$};
            
            \draw[fill=yellow!20, draw=black] (0,0) rectangle (1.5,2);
            \node at (0.75,0.5) {$t$};
            \draw[fill=yellow!20, draw=black] (0.6,1.1) rectangle (1.4,1.9);
            \node at (1,1.5) {$m$};
            \draw[fill=yellow!20, draw=black] (8,2) rectangle (9,0);
            \node at (8.5,1) {$d$};
            \node[fill=green!25,draw=black,circle] at (7.5,3.2) {$a$};
    
            \draw[draw=black!20] (2,0) -- (2,4);
            \draw[draw=black!20] (3.5,0) -- (3.5,4);
            \draw[draw=black!20] (4,0) -- (5,2.5) -- (6,0);
            \draw[draw=black!20] (5,2.5) -- (5,4);
            \draw[draw=black!20] (6.5,0) -- (6.5,4);
            
            \node at (0.75,3.2) {$p_1$};
            \node at (2.75,3.2) {$p_2$};
            \node at (4.25,3.2) {$p_3$};
            \node at (5,0.8) {$p_4$};
            \node at (5.75,3.2) {$p_5$};
            \node at (7.25,0.8) {$p_6$};
            
            \draw[black, very thick] (0,0) rectangle (9,4);
    
            \draw[black, very thick] (3.5,0) -- (3.5,2);
            \draw[black, very thick] (3.5,3) -- (3.5,4);
            \draw[black, very thick] (6.5,0) -- (6.5,1);
            \draw[black, very thick] (6.5,2) -- (6.5,4);
        \end{scope}
    \end{tikzpicture}
    \caption{Environment $e$}\label{fig:example_1_a}
    \end{subfigure}
    \hfil
    \begin{subfigure}{0.65\linewidth}
    \centering
    \begin{tikzpicture}
      [level distance=9mm,
       every node/.style={fill=yellow!80,circle,draw,minimum size=6mm,scale=0.8},
       level 1/.style={sibling distance=20mm,nodes={fill=yellow!60}},
       level 2/.style={sibling distance=8mm,nodes={fill=yellow!40}},
       level 3/.style={sibling distance=8mm,nodes={fill=yellow!20}},
       dot/.style={circle, fill, minimum size=6pt, inner sep=0pt, fill=black!70, draw=none},
       edge from parent/.style={->,draw}]
      \node {$e$}
         child {node (r1) {$r_1$}
            child {node (p1) {$p_1$} 
                child {node (m) {$m$}}
                child {node (t) {$t$}}
            }
            child {node (p2) {$p_2$}}
         }
         child {node (r2) {$r_2$}
            child {node (p3) {$p_3$}}
            child {node (p4) {$p_4$}}
            child {node (p5) {$p_5$}}
         }
         child {node (r3) {$r_3$}
            child {node (p6) {$p_6$}
                child {node[fill=green!25] (a) {$a$}} 
                child {node (d) {$d$}}
            }
         };
      \draw[double,->] (t) -- (m);
      \draw[dashed] (p1) -- (p2);
      \draw[dashed] (p2) -- (p3);
      \draw[dashed] (p3) -- (p4);
      \draw[dashed] (p4) -- (p5);
      \draw[dashed] (p3) .. controls +(down:6mm) and +(down:6mm) .. (p5);
      \draw[dashed] (p5) -- (p6);
    \end{tikzpicture}
    \caption{Scene graph $G_1$}\label{fig:example_1_b}
    \end{subfigure}
    \caption{The environment from Example \ref{ex:1_scene} (a), and its scene graph $G_1$ (b). \captionedges{}}
    \label{fig:example_1}
    \figvspace{}
\end{figure}

\begin{example}[Scene graph] \label{ex:1_scene}
    Consider the environment $e$ (Fig. \ref{fig:example_1_a}) consisting of three rooms ($r_1$, $r_2$, $r_3$), with both $r_1$ and $r_3$ only adjacent to $r_2$ (\ie{} an open door exists between these rooms). Room $r_1$ is subdivided into two places ($p_1$, $p_2$), room $r_2$ in three places ($p_3$, $p_4$,  $p_5$), while room $r_3$ is composed of only one place ($p_6$). Places are adjacent to each other according to Fig. \ref{fig:example_1_a}. Currently, $p_1$ contains a mug $m$ onto a table $t$, while $p_6$ contains a desk $d$ and the robotic agent $a$. The scene graph $G_1$ for this environment is shown in Fig. \ref{fig:example_1_b}.
    The spatial and semantic features of elements in $V_o(G_1)$ (\eg{} pose, bounding box, semantic class) are omitted from this example for brevity.
\end{example}

\subsection{Scene graph transition system}\label{sec:sg_ts}

Using scene graphs as its states, a robot-environment system is modeled as a deterministic transition system.

\begin{definition}[Scene Graph Transition System] Let $\G$ be the set of all possible scene graphs for a given robot-environment system satisfying that for any $G, G' \in \G$, $V(G)=V(G')$. A scene graph transition system is the tuple $\X = (\G, U(\G), f)$, in which $U(\G)$ is the set of actions consistent with the scene graph representation and $f: \G \times U(\G) \rightarrow \G$ is the state transition function.
\end{definition}

To formally define the scene graph transition function, we first introduce a set of actions. Note that different sets of actions are possible depending on the scene graph representation and robot embodiment. However, we assume that all action sets are parametrized by the graph nodes and encode possible ways a robot can interact with its environment. Interactions could be for example, navigation, manipulation, and inspection. In the following, we introduce the set of actions that we will refer to in the examples used throughout the paper.
\begin{definition}[Action set of $\G$]\label{def:action_set}
    Let $\G$ be the set of all scene graphs, let $U_\move (\G)$, $U_\picktext(\G)$, and $U_\placeontext(\G)$ be the sets of actions parametrized by the graph nodes such that
    \begin{align*}
        U_\move (\G)       & := \{\moveto(v_{p_i}) \mid {v_{p_i}\in V_p(\G)}\},                  \\
        U_\picktext(\G)    & := \{\pick(v_{o_i}) \mid {v_{o_i} \in V_o(\G)}\},                   \\
        U_\placeontext(\G) & :=\{\placeon(v_{o_i},v_{o_j}) \mid {v_{o_i},v_{o_j} \in V_o(\G)}\},
    \end{align*}
    in which $V_p(\G)$ and $V_o(\G)$ are the sets of all place nodes and object nodes, respectively. Then, the set of actions for a space of scene graphs $\G$ is
    $U(\G) := U_\move (\G) \cup  U_\picktext(\G) \cup U_\placeontext(\G).$
\end{definition}

The set of actions given in Definition~\ref{def:action_set} are interpreted over the scene graphs in the following ways.
\subsubsection{Move action}
Let $G$ be a scene graph and let $G':=f(G, \moveto(v_p'))$. Let $v_p$ correspond to the place the robot is in, that is, $(v_p, v_a)_{\inc}\in E(G)$, in which $v_a$ is the node representing the robot. If $(v_p, v_{p'})_{adj} \in E(G)$, applying the action $\moveto(v_{p'})$ removes the inclusion edge $(v_{p}, v_a)_\inc$ and adds $(v_{p'}, v_a)_\inc$. Similarly, for any $v_o$ satisfying that $(v_a,v_o)_\con \in E(G)$, the inclusion edge $(v_{p}, v_o)_\inc$ is removed and $(v_{p'}, v_o)_\inc$ is added. Moving an object from one place to another might change whether two places are path connected. The adjacency relations, consequently the scene graph edges, are updated accordingly with the move of an object.
If $(v_p, v_{p'})_{adj} \not\in E(G)$, $(G, \moveto(v_{p'})) \mapsto G$ under $f$.

\subsubsection{Pick action}
A pick action takes an object node as a parameter. Applying $\pick(v_{o})$ action, in which $v_o$ is an object node in place $v_p$, adds the carrier edge $(v_a,v_{o})_\con$ to $E(G)$ if $(v_p, v_o)_\inc, (v_p, v_a)_\inc \in E(G)$, that is the object $o$ is picked up by the robot if they are in the same place, and if the cardinality of the carrier relations the robot is in does not exceed a predefined number. Furthermore, any previous carrier edge that involves $v_o$ is removed. Otherwise, $(G, \pick(v_{o})) \mapsto G$ under $f$.

\subsubsection{Placeon action}
A $\placeon(\cdot)$ action is parametrized by a pair of object nodes $(o_i, o_j)$. Assume $G$ contains a carrier edge $(v_a,v_{o_i})_\con$, and inclusion edges $(v_p, v_a)_\inc, (v_p, v_{o_i})_\inc, (v_p, v_{o_j})_\inc \in E(G)$ for some place node $v_p$, \ie{} the agent and both $o_i$ and $o_j$ are in the same place. Then, applying $\placeon(v_{o_i},v_{o_j})$ removes the carrier edge $(v_a,v_{o_i})_\con$ from $G$ and adds the carrier edge $(v_{o_j},v_{o_i})_\con$.
Otherwise, after applying this action the state remains the same.

\begin{example}[Bringing a mug to the desk] \label{ex:2_actions}
    Considering the starting scene graph $G_1$ from Example \ref{ex:1_scene} (shown in Fig. \ref{fig:example_1}), let us describe the transitions corresponding to a sequence of actions moving the robotic agent $a$ to room $r_1$, picking up the mug $m$, and placing it on the desk $d$ in $r_3$.
    \begin{enumerate}[label=(\alph*)]
        \item Starting in $p_6$, the robot $a$ first executes a sequence of $\moveto$ actions to reach $p_1$. Given that each $\moveto$ action changes only the inclusion relations between $a$ and nodes in $V_p$, the combined graph transition is
              \begin{align*}
                  f(f( & f(f(G_1, \moveto(p_5)), \moveto(p_3)), \\&\moveto(p_2)), \moveto(p_1)) = G_2
              \end{align*}
              with $E(G_2) = (E(G_1) \setminus (p_6,a)_\inc) \cup (p_1,a)_\inc$.
        \item The robot then executes a $\pick(m)$ action. The transition is $f(G_2, \pick(m)) = G_3$ with $E(G_3) = (E(G_2) \setminus (t,m)_\con) \cup (a,m)_\con$.
        \item The robot then executes a sequence of $\moveto$ actions to go back to $p_6$. The transition
              \begin{align*}
                  f(f( & f(f(G_3, \moveto(p_2)), \moveto(p_3)), \\&\moveto(p_5)), \moveto(p_6)) = G_4
              \end{align*}
              with $E(G_4) = (E(G_3) \setminus \{(p_1,a)_\inc,(p_1,m)_\inc\}) \cup \{(p_6,a)_\inc,(p_6,m)_\inc\}$.
        \item Finally, the robot executes a $\placeon(m,d)$ action. The graph transition $f(G_4, \placeon(m,d)) = G_5$ with $E(G_5) = (E(G_4) \setminus (a,m)_\con) \cup (d,m)_\con$.
    \end{enumerate}
\end{example}


\section{Sufficient Scene Graphs for Planning}
When a scene graph transition system is obtained through some perception pipeline (\eg{} \cite{hughes2022hydra}), we consider this to be a maximal model of the external system (physical world).
In this section, based on the framework introduced in Sec.~\ref{sec:theory}, we present a new perspective to analyze what is necessary from a scene graph given a planning task. To this end, we introduce the notions \emph{derived scene graph} and \emph{derived scene graph transition system} and discuss the conditions under which such a system is sufficient for planning.

\subsection{Derived Scene Graphs}\label{sec:derived_sg}
Let $\G$ be the set of all possible maximal scene graphs for a given robot-environment system. We will refer to $\G$ as the \emph{maximal scene graph space}. A \emph{derived scene graph space} $\G_\subder$ is then the codomain of some many-to one I-map $\imap: \G \rightarrow \G_\subder$.
We want $\G_\subder$, the codomain of $\imap$, to also be a set of scene graphs, \ie{} its elements satisfy Definition~\ref{def:scene_graph}.

The I-map $\imap$ is defined through merge and prune operations over the elements of $\G$. Specifically, we consider merging nodes in $V_p$ and $V_r$ (nodes corresponding to places and rooms) and pruning nodes in $V_o$ (nodes corresponding to objects).
To merge nodes of $G$, we define two equivalence relations respectively between the nodes in $V_p$ and in $V_r$ denoted by $\R_p \subseteq V_p \times V_p$ and $\R_r \subseteq V_r \times V_r$.
Nodes that are not merged with any other node form singleton equivalence classes.
Denote with $\overline{V}_o \subseteq V_o$ the set of object nodes to be pruned. Building nodes, robot node, and non-pruned object nodes are also treated as singleton equivalence classes. Let $\R$ denote the resulting equivalence relation over $V(G)$.
Then, the derived graph $G'=\imap(G)$ has node set $V(G') = \{ [v]_\R \mid v\in V(G)\setminus \overline{V}_o \}$.
We identify each node in $G'$ with its equivalence class.
Its edge set is obtained by mapping the endpoints of all edges not incident to pruned object nodes:
$$E(G')=\{ ([u]_\R , [v]_\R) \mid (u,v) \in E(G) \land u,v \not\in \overline{V}_o\}.$$
Edge labels are inherited from the original relation that generated them and self-loops created by merging adjacent nodes are discarded.
To ensure that the image of $\imap$ is a scene graph according to Definition~\ref{def:scene_graph} we require that
\begin{multline}\label{eq:sg_merge_constraint}
    (v_{r_i},v_{p_j})_\inc,(v_{r_k},v_{p_l})_\inc \in E(G) \land (v_{p_j},v_{p_l})\in \R_p \\ \implies (v_{r_i},v_{r_k})\in \R_r\enspace.
\end{multline}
This condition ensures that each node has a unique parent connected with an inclusion edge. Second, for every equivalence class of place nodes induced by $\R_p$, the subgraph induced by that class under $\R_\adj$ is connected. This ensures that each derived place node corresponds to a connected region of free configuration space.

\begin{figure}
    \centering
    \begin{subfigure}{.3\linewidth}
    \centering
    \begin{tikzpicture}
      [level distance=8mm,
       every node/.style={fill=yellow!80,circle,draw,minimum size=6mm,scale=0.8},
       level 1/.style={sibling distance=20mm,nodes={fill=yellow!60}},
       level 2/.style={sibling distance=13mm,nodes={fill=yellow!40}},
       level 3/.style={sibling distance=8mm,nodes={fill=yellow!20}},
       dot/.style={circle, fill, minimum size=6pt, inner sep=0pt, fill=black!70, draw=none},
       edge from parent/.style={->,draw}]
      \node {$e$}
         child {node (r123) {$r_{123}$}
            child {node (p12) {$p_{12}$}
                child {node (m) {$m$}}
                child {node (t) {$t$}}
            }
            child {node (p3456) {$p_{3456}$}
                child {node[fill=green!25] (a) {$a$}}  
            }
         };
      \draw[double,->] (t) -- (m);
      \draw[dashed] (p12) -- (p3456);
    \end{tikzpicture}
    \caption{$G'_1$}\label{fig:example_3_a}
    \end{subfigure}
    \hfil
    \begin{subfigure}{.65\linewidth}
    \centering
    \begin{tikzpicture}
      [level distance=9mm,
       every node/.style={fill=yellow!80,circle,draw,minimum size=6mm,scale=0.8},
       level 1/.style={sibling distance=20mm,nodes={fill=yellow!60}},
       level 2/.style={sibling distance=8mm,nodes={fill=yellow!40}},
       level 3/.style={sibling distance=8mm,nodes={fill=yellow!20}},
       dot/.style={circle, fill, minimum size=6pt, inner sep=0pt, fill=black!70, draw=none},
       edge from parent/.style={->,draw}]
      \node {$e$}
         child {node (r1) {$r_1$}
            child {node (p1) {$p_1$}}
            child {node (p2) {$p_2$}
                child {node (m) {$m$}}
                child {node (t) {$t$}}
            }
         }
         child {node (r2) {$r_2$}
            child {node (p3) {$p_3$}}
            child {node (p4) {$p_4$}
                child {node[fill=green!25] (a) {$a$}} 
            }
            child {node (p5) {$p_5$}}
         }
         child {node (r3) {$r_3$}
            child {node (p6) {$p_6$}}
         };
      \draw[double,->] (t) -- (m);
      \draw[dashed] (p1) -- (p2);
      \draw[dashed] (p2) -- (p3);
      \draw[dashed] (p3) -- (p4);
      \draw[dashed] (p4) -- (p5);
      \draw[dashed] (p3) .. controls +(down:5mm) and +(down:5mm) .. (p5);
      \draw[dashed] (p5) -- (p6);
    \end{tikzpicture}
    \caption{$\tilde{G}_1$, an element of $\imap_e^{-1}(G'_1)$}\label{fig:example_3_b}
    \end{subfigure}
    \caption{The derived graph $G'_1$ from Example \ref{ex:3_der} (a), followed by $\tilde{G}_1$ (b).
    \captionedges{}}
    \label{fig:example_3}
    \figvspace{}
\end{figure}

\begin{example}[Preimage of derived scene graph] \label{ex:3_der} Given the environment $e$ from Example~\ref{ex:1_scene}, we define $\G_e$ as the set of all possible maximal scene graphs of $e$. Let $\imap_e: \G_e\rightarrow{\G}_{\subder,e}$ be a mapping from $\G_e$ to a derived space $\G_{\subder,e}$ such that $\imap_e$: merges all rooms into $r_{123}$; merges places $p_1$ and $p_2$ in $p_{12}$; merges places $p_3$, $p_4$, $p_5$, and $p_6$ into $p_{3456}$; and prunes the desk node $d$, if present. Let the scene graph $G'_1 \in \G_{\subder,e}$, shown in Fig. \ref{fig:example_3_a}, be a derived scene graph under $\imap_e$. When considering the preimage of $G'_1$, \ie{} $\imap_e^{-1}(G'_1) \subseteq \G_e$, we can determine, for example, that both $G_1$ from Fig. \ref{fig:example_1_b} and $\tilde{G}_1$ from Fig. \ref{fig:example_3_b} are elements of $\imap_e^{-1}(G'_1)$ since $\imap_e(G_1) = \imap_e(\tilde{G}_1) = G'_1$. On the other hand, none of the graphs from Example~\ref{ex:2_actions} is in $\imap_e^{-1}(G'_1)$, due to $a$ being in $p_1$ in $G_2$, $G_3$, or $m$ in $p_6$ in $G_4$, $G_5$.
\end{example}

The action set $U(\G_\subder)$ is the set of actions according to Definition~\ref{def:action_set}, meaning that they are parametrized by the nodes of $V(\G_\subder)$. Because $\imap$ modifies the nodes of the scene graphs in $\G$, $U(\G_\subder) \not=U(\G)$.
To illustrate this issue, suppose under $\imap$, an object node $v_{o}$ was pruned. Then, $\pick(v_{o})$, $\placeon(v_{o}, \cdot)$, $\placeon(\cdot, v_{o})\notin U(\G_\subder)$. Similarly, $\moveto(v_{p}) \notin U(\G_\subder)$ if $v_{p}\notin V(\G_\subder)$ due to merging. Every derived action is interpreted as a motion primitive (Sec.~\ref{sec:motion_prim}) over the maximal system. Whereas a motion primitive $\moveto(\cdot)$ may correspond to a sequence in $U(\G)$ of length greater than one, retained $\pick(\cdot)$ and $\placeon(\cdot)$ actions correspond to sequences of length one.

For a maximal scene graph $G_1$, let ${p_i}$ denote the place node containing the robot in $G_i$, \ie{}, $({p_i},v_a)_\inc \in E(G_i)$. Consider a motion primitive $\moveto(\omega)$, in which $\omega \in V_p/\R_p$ and $V_p$ is the set of place nodes in the maximal scene graph. Its induced trace $(G_1, u_1, \dots, u_{N-1},G_N)$ in the maximal system satisfies that every $u_i$ is a move action over $\G$ and that $p_N \in \omega$. Furthermore, there exists an index $m \in \{0,\dots,N\}$ such that $p_i \in [p_1]_{\R_p}$ for all $i<m$ and $p_i \in \omega$ for all $i\geq m$. Therefore, the robot remains inside the initial equivalence class until it enters the target class, and once it enters the target class it does not leave it.

\begin{figure*}
    \centering
    \begin{subfigure}{0.21\textwidth}
    \centering
    \begin{tikzpicture}
      [level distance=10mm,
       every node/.style={fill=yellow!80,circle,draw,minimum size=6mm,scale=0.8},
       level 1/.style={sibling distance=22mm,nodes={fill=yellow!60}},
       level 2/.style={sibling distance=14mm,nodes={fill=yellow!40}},
       level 3/.style={sibling distance=7mm,nodes={fill=yellow!20}},
       dot/.style={circle, fill, minimum size=6pt, inner sep=0pt, fill=black!70, draw=none},
       edge from parent/.style={->,draw}]
      \node {$e$}
         child {node (r123) {$r_{123}$}
            child {node (p123456) {$p_{123456}$}
                child {node (m) {$m$}}
                child {node (t) {$t$}}
                child {node[fill=green!25] (a) {$a$}}  
            }
         };
      \draw[double,->] (t) -- (m);
    \end{tikzpicture}
    \caption{$\imap_1(G_1)$}\label{fig:example_4_a}
    \end{subfigure}
    \begin{subfigure}{0.21\textwidth}
    \centering
    \begin{tikzpicture}
      [level distance=10mm,
       every node/.style={fill=yellow!80,circle,draw,minimum size=6mm,scale=0.8},
       level 1/.style={sibling distance=22mm,nodes={fill=yellow!60}},
       level 2/.style={sibling distance=14mm,nodes={fill=yellow!40}},
       level 3/.style={sibling distance=7mm,nodes={fill=yellow!20}},
       dot/.style={circle, fill, minimum size=6pt, inner sep=0pt, fill=black!70, draw=none},
       edge from parent/.style={->,draw}]
      \node {$e$}
         child {node (r123) {$r_{123}$}
            child {node (p123456) {$p_{123456}$}
                child {node (m) {$m$}}
                child {node (t) {$t$}}
                child {node[fill=green!25] (a) {$a$}}  
                child {node (d) {$d$}}
            }
         };
      \draw[double,->] (t) -- (m);
    \end{tikzpicture}
    \caption{$\imap_2(G_1)$}\label{fig:example_4_b}
    \end{subfigure}
    \begin{subfigure}{0.26\textwidth}
    \centering
    \begin{tikzpicture}
      [level distance=10mm,
       every node/.style={fill=yellow!80,circle,draw,minimum size=6mm,scale=0.8},
       level 1/.style={sibling distance=22mm,nodes={fill=yellow!60}},
       level 2/.style={sibling distance=13mm,nodes={fill=yellow!40}},
       level 3/.style={sibling distance=7mm,nodes={fill=yellow!20}},
       dot/.style={circle, fill, minimum size=6pt, inner sep=0pt, fill=black!70, draw=none},
       edge from parent/.style={->,draw}]
      \node {$e$}
         child {node (r123) {$r_{123}$}
            child {node (p1) {$p_{1}$}
                child {node (m) {$m$}}
                child {node (t) {$t$}}
            }
            child {node (p2345) {$p_{2345}$}}
            child {node (p6) {$p_{6}$}
                child {node[fill=green!25] (a) {$a$}}  
                child {node (d) {$d$}}
            }
         };
      \draw[double,->] (t) -- (m);
      \draw[dashed] (p1) -- (p2345);
      \draw[dashed] (p6) -- (p2345);
    \end{tikzpicture}
    \caption{$\imap_3(G_1)$}\label{fig:example_4_c}
    \end{subfigure}
    \begin{subfigure}{0.29\textwidth}
    \centering
    \begin{tikzpicture}
      [level distance=8mm,
       every node/.style={fill=yellow!80,circle,draw,minimum size=6mm,scale=0.8},
       level 1/.style={sibling distance=22mm,nodes={fill=yellow!60}},
       level 2/.style={sibling distance=10mm,nodes={fill=yellow!40}},
       level 3/.style={sibling distance=8mm,nodes={fill=yellow!20}},
       dot/.style={circle, fill, minimum size=6pt, inner sep=0pt, fill=black!70, draw=none},
       edge from parent/.style={->,draw}]
      \node {$e$}
         child {node (r123) {$r_{123}$}
            child {node (p1) {$p_1$}
                child {node (m) {$m$}}
                child {node (t) {$t$}}
            }
            child {node (p235) {$p_{235}$}}
            child {node[fill=red!20] (p4) {$p_4$}}
            child {node (p6) {$p_6$}
                child {node[fill=green!25] (a) {$a$}}  
                child {node (d) {$d$}}
            }
         };
      \draw[double,->] (t) -- (m);
      \draw[dashed] (p1) -- (p235);
      \draw[dashed] (p4) -- (p235);
      \draw[dashed] (p235) .. controls +(down:5mm) and +(down:5mm) .. (p6);
    \end{tikzpicture}
    \caption{$\imap_4(G_1)$}\label{fig:example_4_d}
    \end{subfigure}
    \caption{The sequence of derived scene graphs presented in Example \ref{ex:4_task}. \captionedges{}}
    \label{fig:example_4}
    \figvspace{}
\end{figure*}

\begin{definition}[Derived scene graph transition system] Given a scene graph transition system $(\G,U(\G),f)$, and a mapping $\imap:\G \rightarrow \G_\subder$, let $U(\G_\subder)$ be the set of actions over $\G_\subder$ according to Definition~\ref{def:action_set}. Then, a scene graph transition system derived from $(\G,U(\G),f)$ by $\imap$ is the quotient system $(\G_\subder, U(\G_\subder), \F_\subder):=(\G,U(\G_\subder),\diamond_f)/\imap$.
\end{definition}

Note that $(\G_\subder, U(\G_\subder), \F_\subder)$ is deterministic if $\imap$ is sufficient with respect to $\diamond_f$, \ie{} it satisfies \eqref{eq:sufficiency}.

\subsection{Scene graphs sufficient for planning}

In Sec.~\ref{sec:theory}, we formally defined a model sufficient for planning for a task. Here, we apply this notion to the case when models are derived scene graph transition systems.

Let $\X=(\G, U(\G), f)$ be the maximal scene graph transition system and $T \subseteq \Omega_\X$ be a task over its traces.
Given an I-map $\imap: \G \rightarrow \G_\subder$, let $\X_{U(\G_\subder)}:=(\G,U(\G_\subder),\diamond_f)$ be the restriction of $\X$ by the set $U(\G_\subder)$ of motion primitives. Consecutively, let $\X_\subder$ be the derived scene graph transition system, that is, $\X_\subder:= \X_{U(\G_\subder)}/\imap$.

Following Sec.~\ref{sec:motion_prim}, the motion-primitive condenser is $c_\MP: \bar{\Omega}_{\X_U(\G_\subder)} \rightarrow {\Omega}_{\X_U(\G_\subder)}$. The I-map $\imap$ induces the condenser $c_\imap: {\Omega}_{\X_U(\G_\subder)} \rightarrow \Omega_{\X_\subder}$ from the traces corresponding to motion-primitives to derived scene graph traces as given in \eqref{eq:cImap}. Then, the complete condenser is $c=c_\imap \circ c_\MP$.

\begin{proposition}\label{prop:SG_suff}
    Given a maximal scene graph transition system $\X=(\G, U(\G), f)$, let $\X_{U(\G_\subder)}=(\G,U(\G_\subder),\diamond_f)$ be its restriction by motion primitives and let $\X_\subder= \X_{U(\G_\subder)}/\imap$ be the derived scene graph transition system by $\imap$. Let $T \subseteq \Omega_\X$ be a task and define $T_{U(\G_\subder)}:=T \cap \bar{\Omega}_{\X_U(\G_\subder)}$. Suppose $\imap$ is sufficient with respect to $\diamond_f$ and $T_{U(\G_\subder)}$ is well-posed under $c=c_\imap\circ c_\MP$. Then, every feasible plan in $\X_\subder$, initialized at $\imap(G_1)$ for $c(T_{U(\G_\subder)})$ is feasible in $\X$ for $T$ when executed from $G_1 \in \G$.
\end{proposition}

\begin{proof}
    By construction, every trace of $\X_{U(\G_\subder)}$ corresponds to a trace of $\X$ under $c_\MP$. Hence, well-posedness of $T_{U(\G_\subder)}$ under $c$ implies that $c_\MP(T_{U(\G_\subder)})$ is well-posed under $c_\imap$. Then we can apply Proposition~\ref{prop:poseable_feasible} to $\X_{U(\G_\subder)}$ with task $c_\MP(T_{U(\G_\subder)})$. Every feasible plan in $\X_\subder$ induces a trace in $\X_{U(\G_\subder)}$ belonging to $c_\MP(T_{U(\G_\subder)})$ whose corresponding trace in $\X$ belongs to $T_{U(\G_\subder)}$ by well-posedness under $c$. Since $T_{U(\G_\subder)} \subseteq T$, the plan corresponding to the sequence of actions in $U(\G)$ encoded by the motion primitive sequence is feasible for $T$.
\end{proof}

\begin{corollary}Let the assumptions of Proposition~\ref{prop:SG_suff} hold. If the set $U(\G_\subder)$ is complete for $T$ from the admissible initial states $\G_0 \in \G$, then for any $G_1 \in \G_0$, existence of a feasible plan for $T$ in $\X$ from $G_1$ implies existence of a feasible plan for $c(T_{U(\G_\subder)})$ in $\X_\subder$ from $\imap(G_1)$.
\end{corollary}

\begin{proof}
    If a feasible trace for $T$ exists from $G_1$, completeness of motion primitives implies that there exists a motion primitive generated trace from $G_1$ belonging to $T_{U(\G_\subder)}$. Then, $c$ maps this trace to a trace in $\X_\subder$ belonging to $c(T_{U(\G_\subder)})$, proving that a feasible plan exists in $\X_\subder$.
\end{proof}

The following example illustrates the notions of sufficient scene graphs and when the required conditions are violated.

\begin{example}[Sufficient scene graph for a task]\label{ex:4_task} Consider again the environment $e$ from Example \ref{ex:1_scene}, represented by $\X=(\G_e, U(\G_e),f)$. Let the task $T_e \subseteq \Omega_\X$ be described informally as ``Bringing a mug to the desk while avoiding $p_4$''. $T_e$ is the set of all finite scene graph traces where the robot $a$ is never in $p_4$, and which end with a graph where the mug $m$ is onto desk $d$; formally:
    \begin{align*}
        T_e = \{ (G_1, u_1, G_2, \dots, G_{N}) \in \Omega_\X \mid (d,m)_{\con} \in E(G_{N}) \\
        \wedge (\not\exists G \in (G_1,\dots,G_N)~(p_4,a)_\inc \in E(G))\}_{N=0}^{<\Nat}
    \end{align*}
    Suppose $G_1$ from Example \ref{ex:1_scene} is the unique initial state for $\X$.
    We now consider the derived scene graphs in Fig.~\ref{fig:example_4} and discuss which information must be preserved for the task.

    Let $\imap_1$ be the I-map whose derived graph $\imap_1(G_1)$ is shown in  Fig.~\ref{fig:example_4_a}. Since the node $d$ is pruned, the derived graph cannot distinguish traces that end with $(d,m)_{\con} \in E(G_{N})$ from traces that do not. Therefore, the task is not well-posed under $c_1 = c_{\imap_1}\circ c_\MP$.

    Let $\imap_2$ be the I-map whose derived graph $\imap_2(G_1)$ is shown in  Fig.~\ref{fig:example_4_b}. This I-map preserves the node $d$ distinguishing the traces that satisfy the goal condition from the ones that do not. However, since all places are merged into a single node in the derived graph, the outcomes of $\pick(m)$ and $\placeon(m,d)$ actions cannot be determined in the derived system. Hence $\imap_2$ is not sufficient.

    Let $\imap_3$ be the I-map whose derived graph $\imap_3(G_1)$ is shown in  Fig.~\ref{fig:example_4_c}. This mapping distinguishes $p_1$, and $p_6$. Therefore, the relevant $\pick(m)$ and $\placeon(m,d)$ outcomes are uniquely determined. However, the places are merged into a single derived place node $p_{2345}$. Consequently, a derived trace using $\moveto(p_{2345})$ cannot determine whether the corresponding traces in the maximal system pass through $p_4$. The task is not well-posed under $c_3=c_{\imap_3}\circ c_\MP$.

    Let $\imap_4$ be the I-map whose derived graph $\imap_4(G_1)$ is shown in  Fig.~\ref{fig:example_4_d}. This mapping distinguishes $p_1$, $p_4$, and $p_6$, while merging the rest of the place nodes. The derived graph retains the information needed to determine the outcomes of $\pick(m)$ and $\placeon(m,d)$, and it also preserves the information needed to check whether the robot visits $p_4$. Thus, on the reachable part of the transition system considered in this example, $\imap_4$ satisfies the sufficiency and well-posedness conditions for $T_e$. The derived plan
    \begin{align*}
        \pi_4 = ( & \moveto(p_{235}),\moveto(p_1),\pick(m),      \\
                  & \moveto(p_{235}),\moveto(p_6),\placeon(m,d))
    \end{align*}
    corresponds to a feasible plan for $T_e$ in the maximal scene graph transition system.
\end{example}


\section{Discussion and Conclusions}
We introduced a formal notion of sufficient scene graphs for robotic task planning. The proposed formalism reframes scene graph compression as a task-dependent information-abstraction problem. Rather than asking whether a derived scene graph preserves all properties of the maximal graph, the relevant question is whether it preserves the information needed to find plans that remain feasible in the original system. In this view, scene graphs are states of a transition system over which pruning and merging operations define an I-map, and derived scene graphs are quotient models induced by that mapping.

The resulting conditions separate distinct sources that allow checking feasibility of plans obtained over derived scene graphs. Deterministic quotient behavior ensures that derived actions have well-defined effects, while well-posedness ensures that task satisfaction is determined by derived traces. Together, these conditions guarantee that plans obtained over derived systems lift to feasible plans over the maximal system. Preservation of plan existence additionally requires completeness of the motion primitives induced by the derived action set.

This perspective is complementary to existing scene graph planning methods. Hierarchical, heuristic, and LLM-guided compression methods can be seen as implicitly choosing an I-map, and the framework presented here provides criteria for analyzing such choices. The characterization of planning sufficiency also suggests computational directions. Given a candidate merge and prune operations, sufficiency can be checked by searching counterexamples (similar to~\cite{seipp2018counterexample}), that is, searching for traces in the maximal system that map to the same derived trace but have different task labels. Furthermore, it also connects to partition-refinement methods for computing quotient systems since the induced I-map labeling is required to be a refinement of task labeling.

The present formulation is limited to deterministic scene graph transition systems, open-loop plans, and a restricted set of relations. Extending the framework to uncertainty, richer semantic relations, and developing algorithms for constructing (near) minimal sufficient derived scene graphs or their approximations remain important directions for future work.


\section*{Acknowledgments}
OpenAI's GPT5.5 was used to improve text readability.


\bibliographystyle{IEEEtran}
\bibliography{IEEEabrv,bibliography}
\balance
\end{document}